%% file: corba_arxiv_v2.tex
\documentclass{article} 
\usepackage{macros}
\usepackage{hyperref}

\usepackage{times}
\newcommand\blfootnote[1]{%
  \begingroup
  \renewcommand\thefootnote{}\footnote{#1}%
  \addtocounter{footnote}{-1}%
  \endgroup
}

\title{Correlated bandits or:\\ How to minimize mean-squared error online}

\author{
Vinay Praneeth Boda\\
{\normalsize LinkedIn Corp.\blfootnote{
Vinay Praneeth Boda was with the University of Maryland, College Park, and a  portion of this work was done while the authors were at the University of Maryland, College Park.}} \\
{\normalsize \texttt{vpboda@gmail.com}}
\and 
Prashanth L. A.\\
{\normalsize Indian Institute of Technology Madras}\\
{\normalsize \texttt{prashla@cse.iitm.ac.in}}
}
\date{}
\begin{document}

\maketitle

\begin{abstract}
While the objective in traditional multi-armed bandit problems is to find the arm with the highest mean, in many settings, finding an arm that best captures information about other arms is of interest. This objective, however, requires learning the underlying correlation structure and not just the means of the arms. Sensors placement for industrial surveillance and cellular network monitoring are a few applications, where the underlying correlation structure plays an important role. Motivated by such applications, we formulate the \textit{correlated bandit} problem, where the objective is to find the arm with the lowest mean-squared error (MSE) in estimating all the arms. To this end, we derive first an MSE estimator, based on sample variances and covariances, and show that our estimator exponentially concentrates around the true MSE. Under a best-arm identification framework, we propose a successive rejects type algorithm and provide bounds on the probability of error in identifying the best arm. Using minmax theory, we also derive fundamental performance limits for the correlated bandit problem. 
\end{abstract}

\section{Introduction}
\label{sec:intro}
The traditional multi-armed bandit problem aims to find the arm with the highest payoff. This is often motivated by practical applications such as to identify an ad with highest payoff in showing to users, or identifying a strategy with maximum payoff. In this work, we consider a setting with the objective being the identification of an arm/node which best captures the entire information of a system, i.e., the identification of arm which can best estimate all the other arms. In contrast to the traditional multi-armed bandit problem, this objective involves an estimation of the correlation structure among the various arms. This is motivated by several practical applications. For instance, in internet-of-things, sensors are used to take measurements from multiple locations with the objective of estimating the underlying parameter, e.g., temperature, over a region. Resource constraints mean that it might not possible to place sensors at the desired level of granularity. However, an estimate of the underlying distribution enables one to form an estimate of the parameter at points not measured.
This estimate of the statistics of the underlying randomness is often formed using limited measurements from multiple points, before choosing the final location of the sensors. Another application of interest is in identifying members who can best approximate the social network.
Instances include sensors used for measuring temperature in a region \citep{guestrin2005near}, thermal sensors on microprocessors \citep{long2008thermal}, optimizing queries over a sensornet \citep{deshpande2004model} and placing sensors to detect contaminants in a water distribution network \citep{krause08efficient}. Problems of similar interest have also been studied in the realm of information theory in \citep{Boda19, BodaNar18}. In all these applications, the underlying correlation structure plays an important role. 

In this paper, we formulate a variant of the stochastic $K$-armed bandit problem, where the objective is to identify the arm that best estimates all the other correlated arms. We measure how good an arm $i \in \{ 1, \ldots, K \}$ can estimate other arms using the mean-squared error (MSE) criterion: 
\begin{align}
  \cE_i \triangleq \sum \limits_{j=1}^{K} \mathbb{E} \left [ \big (X_j - \mathbb{ E}[X_j | X_{i}] \big )^{2} \right ]. 
	\label{eq:mmse-intro}
 \end{align}

 We assume that the arms $X_1,\ldots,X_K$ are correlated sub-Gaussian random variables (r.v.s). \cite{utpalcellnetwork} consider a celluar network application, where the goal is to monitor large communication networks with huge traffic. 
Since observing every node is computationally intensive, companies such as AT\&T use  measurements from various nodes to identify a subset which best captures the average behavior of the network. The requirement is for an algorithm that reduces the data acquisition cost by identifying the most-correlated subset of nodes, while using a minimum number of sample measurements. 
The authors in \citep{utpalcellnetwork} show that a model approximating the underlying nodes as Gaussian r.v.s
 is useful and reliable. 

 Closely related problems in other application contexts include (i) selecting a few blogs that capture the information cascade \citep{Krause_outbreak}; (ii)
 finding a subset of people who best represent the average behavior of a community; 
 To put it differently, the notions of \textit{centrality} in the context of document/news summarization \citep{erkan2004lexrank} and \textit{prestige} in social networks \citep{heidemann2010identifying} 
 are closely related to the MSE objective in \eqref{eq:mmse-intro}. In each of these applications, there is a cost associated with acquiring data and the challenge is to  find the most correlated subset of blogs/people/etc using minimal observations about the community.

We study the basic problem of identifying the arm which has the best MSE in estimating the remaining arms in a multi-armed bandit framework.
We consider the best arm identification setting \citep{audibert2010best,kaufmann2015complexity}, where a bandit algorithm is given a fixed sampling budget, and is evaluated based on the probability of incorrect identification. 
Challenges encountered for such a setup include:\\
(i) Any estimate for the MSE requires estimation of the underlying correlations, without assuming knowledge of the variances. \\
(ii) Estimate of the MSE of an arm $i$ involves estimating the correlation of arm $i$ with the remaining arms. This requires samples from all pairs of arms associated with $i$. In particular, sampling arm $i$ alone would be insufficient towards estimating arm $i$'s MSE; and hence\\
(iii) A bandit algorithm needs to optimize sampling across all pairs of arms and not just among arms. This requires
  intricate decisions over a larger set, in contrast to the classical mean-value optimizing algorithms in a best arm identification framework.

We summarize our contributions below.

First, we introduce a {\it new formulation} to study the identification of arm which best estimates all arms. 
We design an estimate and develop the concentration bound for the estimate of mean-squared error formed from available samples. Our estimator builds on the difference estimator introduced in \citep{BubeckLiu}, but estimation is technically more challenging in our setting as the underlying variances are not known and unlike \cite{BubeckLiu}, not necessarily assumed to be one. 

Second, 
we analyze a nonadaptive uniform sampling strategy (i.e., a strategy that pulls each pair of arms an equal number of times) and propose an algorithm inspired by popular successive rejects (SR) \citep{audibert2010best} for best-arm identification, but more intricate due to the nonlinearity of the objective function, the MSE objective function \eqref{eq:mmse-intro}. A naive SR strategy that operates over phases, discarding all arm pairs associated with the arm having lowest empirical MSE is suboptimal. Instead, our SR algorithm maintains active sets for arms as well as pairs and discards a pair only if both constituent arms are out of the active arms set. 
We provide an upper bound on the probability of error in identifying the best arm for our SR algorithm and the bound involves a hardness measure that factors in the gaps in MSEs as well as the correlations, which are specific to the correlated bandit problem. As in the classic bandit setup, 
the upper bound shows that SR algorithm requires fewer samples to find the best arm in comparison to a uniform sampling strategy, especially, when $K$ is large and the underlying gaps (difference between MSE of optimal and suboptimal arms) are uneven. 

Third, we prove a lower bound over all bandit problems with a certain hardness measure and to the best of our knowledge, this is the first lower bound for the correlated bandit problem that involves adaptive sampling strategies. 
The lower bound involves constructing problem transformations, where the optimal arm is ``swapped'' with one of the sub-optimal ones, resulting in $K-1$ problem instances. 
Unlike in the classic setup, any local change in the distribution of an arm impacts the MSE of all the other arms. 
Moreover, pulling pairs of arms  instead of individual arms makes the lower bound technically more challenging.

In \citep{BubeckLiu}, which is the closest related work, the authors consider a bandit problem, where the objective is to identify a 
subset of arms most correlated among themselves, i.e., to identify 
the local correlation structure within a subset of arms themselves. On the other hand, our problem is about forming global inference from samples of subsets of arms to identify the arm that is most correlated to the remaining arms.
In \cite{BubeckLiu}, the authors consider a setting with positively correlated arms with unit variance, making the estimation task and hence, the overall best arm identification slightly easier. As we show later in Section \ref{sec:estimation}, their estimation scheme does not extend to the more general non-unit variance setup that we consider. Finally, we also prove fundamental limits on the performance of any correlated bandit algorithm, through information-theoretic lower bounds, and to the best of our knowledge, no lower bounds exist for a correlated bandit problem.

The rest of the paper is organized as follows:
In Section~\ref{sec:model}, we formalize the correlated bandit problem. 
In Section~\ref{sec:estimation}, we present the MSE estimation scheme and derive a concentration bound for our estimator. In Section \ref{sec:uniformsampling}, we examine uniform sampling strategy, while in Section \ref{sec:sr}, we present a successive-rejects type algorithm. In Section \ref{sec:lb}, we present a lower bound for the correlated bandit problem. 
In Section~\ref{sec:proofs}, we provide detailed proofs of the bounds presented in Sections \ref{sec:estimation}--\ref{sec:lb}. While not the thrust of this work, we provide a few illustrative examples in Section~\ref{sec:expts} showing the performance of our successive-rejects type algorithm.
 Finally, in Section~\ref{sec:conclusions} we provide our concluding remarks.
 
 \section{Model}
 \label{sec:model}

We consider a set   $\cM \triangleq \{1,\ldots,K\}$ of $K$ correlated arms $ X_{1}, \ldots, X_{K} $, whose samples are 
i.i.d. in time. 
For each arm $i$, let ${\cal E}_i$ denote the minimum mean-squared error (MMSE) of
$X_i$ estimating all the remaining arms, i.e.,
\begin{align} \label{eq:mmse_g_arms}
 \cE_i \triangleq \underset{g } \min  \  \mathbb{E} \big [ \big (X_{\cM} - g(X_{i}) \big)^T 
 \big(X_{\cM} - g(X_{i}) \big) \big].
\end{align}
Consider the special case of jointly Gaussian r.v.s $ X_{1}, \ldots, X_{K} $, whose joint probability distribution is characterized by the mean (taken to be zero for the sake of expository simplicity), and    
{\it covariance matrix}  $\Sigma  \triangleq  \mathbb{E}[X_{\cM}^{T} X_{\cM}]$:
\begin{align}
\Sigma  = \! \left[ \! \begin{array}{c c c c}
	\sigma^2_1 & \rho_{12}\sigma_1\sigma_2  & \ldots & \rho_{1K}\sigma_1\sigma_K \\
	\rho_{12}\sigma_1\sigma_2  & \sigma^2_2 & \ldots  & \rho_{2K}\sigma_2\sigma_K \\
	\vdots & \vdots & \ddots & \vdots \\
	\rho_{1K}\sigma_1\sigma_K  & \rho_{2K}\sigma_2\sigma_K & \ldots & \sigma^2_K   
\end{array} \! \right]. \!
\label{eq:gauss-covar}
\end{align}
In the above, $\sigma^2_p$, $p \in \cM$ is the variance of arm $p$ and $\rho_{ij}$, $i,j=1,\ldots,K, \ i \neq j,$ the correlation coefficient between arms $i$ and $j$.

The best estimate $g^*$, which achieves the minimum in \eqref{eq:mmse_g_arms}, is known to be the MMSE
estimate. For zero-mean jointly Gaussian r.v.s, this is given by (cf. Chapter 3 of \citep{hajek2009notes})  
\begin{align} 
&g^*(X_{i}) = \mathbb{E}[X_{\cM}|X_i]  = 
\left[\mathbb{E}[X_{1}|X_{i}] \ldots \mathbb{E}[X_{K}|X_{i}] \right]^T,\nonumber\\
&\textrm{ with }
 \mathbb{E}[X_{j}|X_{i}] = \frac{\mathbb{E}[X_{j} X_i]}{\mathbb{E}[X_i^{2}]} X_i
 = \frac{\rho_{ij}  \sigma_j  }{\sigma_i} X_{i}. \label{eq:mmse_def}
\end{align}
The corresponding MMSE for arm $i$ is
\begin{align} \label{eq:MMSE_form}
\!  {\mathcal E}_{i} \!
  = \sum \limits_{j=1}^{K} \mathbb{E} \left [ \big (X_j - \mathbb{ E}[X_j | X_{i}] \big )^{2} \right ] 
  =  \sum \limits_{j \neq i}   \sigma_{j}^{2} ( 1 - \rho_{i j}^{2}   ).
  \vspace{-3ex}
\end{align}
Note that there is no error in arm $i$ estimating itself
and the error in estimating the $j$th arm is characterized 
by the correlation between $X_i$ and $X_j$ and the relevant variances.
Further, the MMSE estimate for the case of Gaussian r.v.s is linear. In the more general case of non-Gaussian r.v.s, the MMSE estimate is typically nonlinear and any online computation is typically a computationally intense task. In such cases, we restrict ourselves to employing an optimal linear estimator which is still defined as the right side of \eqref{eq:mmse_def}. Thus, the right-side of \eqref{eq:MMSE_form} holds for all optimal linear estimators, with it being optimal for Gaussian r.v.s.

We consider a setting where the arms $X_1, \ldots, X_K$ are sub-Gaussian, and focus on linear estimators. We recall the definition of sub-Gaussianity below.
\begin{definition}
 A r.v. $X$ is said to be $\sigma$-sub-Gaussian if 
$\E\left( e^{\lambda X}\right) \leq \exp\left(\frac{\lambda^2 \sigma^2}{2}\right), \quad \forall \lambda \in \R.
$\end{definition}
For equivalent characterizations of sub-Gaussianity, the reader is referred to Theorem 2.1 of \cite{wainwright2019high}.

We consider a fixed budget best-arm identification framework, and the interaction of our (bandit) algorithm with the environment is given below.
\begin{figure}[h] \label{fig:fig1}
\centering
\fbox{
\begin{minipage}{0.95\textwidth}
{\bfseries Correlated bandit algorithm}\ \\[0.5ex]
{\bfseries Input}: set of pairs of arms $\S$, number of rounds $n$. \\
{\bfseries For all $t=1,2,\dots,n$, repeat}
\begin{enumerate}
\item Based on samples $\{ (X_{i_l,l},X_{j_{l},l}), \ l =1, \ldots,t-1 \}$ seen so far, select a pair $(i_t, j_t) \in \S \triangleq\{ (i,j) \mid i,j=1,\ldots,K, i<j\}$.
\item Observe a sample from the bivariate distribution corresponding to the arms $i_t, j_t$.
\end{enumerate}
After $n$ rounds, output an arm $\hat A_n$.
\end{minipage}
}
 \label{fig:flow}
\end{figure}

Notice that, in each round, the algorithm above pulls a pair of arms, and this is necessary to learn the underlying correlation structure.

In our setting, the performance metric associated with each arm $i$ is its MSE $\cE_i$, and 
the optimal arm, say $i^*$, has the lowest MSE, i.e.,
\begin{align*} 
i^* = \underset{i \in \cM} {\arg \min} \ \cE_i.
\end{align*}

The objective is to minimize the 
probability of error in identifying the best arm, i.e.,
\[\Prob{ \hat A_n \ne i^*},\]
where $\hat A_n$ is the estimate of the best arm based on $n$ samples.

For $i \neq i^{*},$ the sub-optimality of the arm $i$ is quantified by 
its gap in its MSE with respect to the optimal arm, i.e., 
$ \Delta_{i} = \cE_i - \cE_{i^{*}}.$
The notation $(i)$ is used to refer to the $i^{\text{th}}$ best arm (with ties broken arbitrarily), i.e.,
$\Delta_{(i)}$s are ordered gaps of the arms:
\[
\Delta_{(1)} \triangleq \Delta_{(2)} \leq \Delta_{(3)} \leq \ldots \leq \Delta_{(K)}.
\]

Note that the problem with $K=2$ reduces to identifying the arm with higher 
variance and has no dependence on the correlation between the arms. The analysis 
of this case would be similar (estimate variance instead of mean) to the classical 
bandit problems and differs considerably from the setting with $K \geq 3$ arms, which is the setting assumed hereafter. 

 \section{MSE Estimation}
 \label{sec:estimation}
Let $\{ (X_{it}, X_{jt}), \ t = 1, \ldots,n \}$ denote the set of $n$ i.i.d. samples obtained from the bivariate Gaussian distribution corresponding to the pair of arms  $(i,j)$. To identify the optimal arm, we form an estimate of ${\cal E}_i$ to which end we form estimates for the variances $\sigma_i^2, \sigma_j^2$ and the correlation coefficient $\rho_{ij}$. We employ the following estimators for the aforementioned quantities: For any  $(i,j) \in {\cal S} = \{ (p,q), 1 \leq p,q \leq K, \ p \neq q \}$,
\begin{align}
&  {\hat \rho}_{ij} \triangleq  1 - \frac{1}{2} \left( \frac{ \overline X_i^2}{\hat\sigma_i^2} 
  + \frac{  \overline X_j^2}{\hat\sigma_j^2} - 
  2 \frac{\overline{X_iX_j}}{\hat\sigma_i\hat\sigma_j} \right),\label{eq:rhohat}
\\
& \hat \sigma_i^2 = {\overline X_i^2}, \ 
 \hat \sigma_j^2 = {\overline X_j^2},  
 \text{~ where ~}\nonumber\\
 &{\overline X_i^2} \! = \! \frac{1}{n} \sum \limits_{t=1}^n X_{i t}^2, \text{ and } \    
 {\overline {X_i X_j}} \!  = \! \frac{1}{n} \sum \limits_{t=1}^n X_{i t} X_{j t}. \nonumber 
\end{align}
The estimate for $\rho$ in \eqref{eq:rhohat} is akin to that proposed in \cite{BubeckLiu}, which considers 
a simpler setting where all the arms are known to have unit variance, i.e.,
$\sigma_i^2 = 1, \ i = 1, \ldots,K.$  
For the unit variance setup, \cite{BubeckLiu} establish via a likelihood ratio test that the difference based estimator for $\rho_{ij}$
\begin{align}
     1 - \frac{1}{2} \big ( {\overline X_i^2} +  {\overline X_j^2} 
 - 2 {\overline {X_i X_j}} \big ) \label{eq:bubeck-est}
\end{align}
is advantageous over the natural estimator for $\rho_{ij}$ :
\[   \frac{ \overline {X_i  X_j} }{ {\hat \sigma}_i {\hat \sigma}_j}.\] 
This superiority depends explicitly on the {\it a priori} knowledge of the variances being one,
which is not applicable to the general setting considered here, i.e., a setting where the variances are not necessarily one.
However, to exploit the optimality of the likelihood ratio test,
we express the estimator above in the spirit of \eqref{eq:bubeck-est} 
which depend on the estimates of the variances to scale
the samples to obtain\\
\[  {\hat \rho}_{ij} = 1 - \frac{1}{2} \left( \frac{ \overline X_i^2}{\hat\sigma_i^2} 
  + \frac{  \overline X_j^2}{\hat\sigma_j^2} - 
  2 \frac{\overline{X_iX_j}}{\hat\sigma_i\hat\sigma_j} \right) ,
\]

Unlike the unit variance setup of \cite{BubeckLiu}, it is not possible to obtain a difference based estimator in our setting.
Nevertheless, $\hat\rho_{ij}$
concentrates faster as $\rho_{ij}$ approaches $1$ and this can be argued as follows:
On the high probability event  ${\cal C} = 
\left \{ \frac{\sigma_1^2}{2} \leq {\hat \sigma}_{1}^2 \leq 2 \sigma_{1}^2, \ 
\frac{\sigma_2^2}{2} \leq {\hat \sigma}_{2}^2 \leq 2 \sigma_{2}^2 \right  \}$, we have
\begin{align*}
&\P \left ( \left ( 1 - \hat\rho_{ij}  \right ) - \left ( 1 - \rho_{ij}  \right )
\geq \epsilon, {\cal C}  \right ) \\
&= \P \left ( \frac{Y_{ijn}}{2n} - 1 \geq \frac{\epsilon}{ \left ( 1 - \rho_{ij}  \right )}, {\cal C}  \right ) \\
& \le \P \left (  \frac{\bar Y_{ijn}}{2n}  - 1  \geq \frac{\epsilon}{2\left ( 1 - \rho_{ij}  \right )}  \right ) \\
&\le \concsubexpTwoterms{n}{\frac{\epsilon}{2\left ( 1 - \rho_{ij}  \right )}}, \\
&\text{ where } Y_{ijn} \triangleq \frac{1}{(1-\rho_{ij})} \left( \frac{\overline X_i^2}{\hat\sigma_i^2} 
 + \frac{\overline X_j^2}{\hat\sigma_j^2} - 2 \frac{\overline{X_iX_j}}{\hat\sigma_i\hat\sigma_j} \right),
\textrm{ and }\\
&\bar Y_{ijn} \triangleq \frac{1}{(1-\rho_{ij})} 
\left( \frac{\overline X_i^2}{\sigma_i^2} + \frac{\overline X_j^2}{\sigma_j^2} 
- 2 \frac{\overline{X_iX_j}}{\sigma_i\sigma_j} \right).
\end{align*}
For any arm $i$, the corresponding MSE $\cE_i$ is estimated 
using the quantities defined in \eqref{eq:rhohat} as follows:
\begin{align} 
& {\hat \cE}_{i} \triangleq {\hat \sigma}_{j}^2 \left ( 1 - \hat\rho_{ij}^2  \right ) 
     + \sum\limits_{p \neq i,j } {\hat \sigma}_{p}^2 \left ( 1 - \hat\rho_{ip}^2 \right ). \label{eq:mse-est-i}
\end{align}    
The main result concerning the concentration of the MSE estimate $\hat \cE_i$ is given below.
\begin{proposition}\textbf{\textit{(MSE Concentration)}}
\label{prop:mse-conc}
Assume $\sigma_i^2 \le 1, i=1,\ldots,K$. Let $\hat\cE_{i}$ be the MSE estimate given in \eqref{eq:mse-est-i}, for $i=1,\ldots,K$.
Then, for any $i=1,\ldots,K$, and for any $\epsilon \in [0,2K]$, we have  
\begin{align*}
 & \P \left (  \left|\hat\cE_{i}  - \cE_{i}\right| > \epsilon   \right) 
\leq 14 K  \exp \left ( - \frac{n l^2 \epsilon^2}{c K^5}  \right),
\end{align*} 
where $c$ is a universal constant, and $0<l = \min\limits_i \ \sigma_i^2$ .
\end{proposition}
In the above, it suffices to look at $\epsilon \leq 2K$, since $\cE_i$ is less than $K-1$, owing to the assumption that $\sigma_i^2 \leq 1, \ \forall i$.
\begin{proof}
See Section \ref{sec:proofs-mse-conc}. 
\end{proof}
The claim in Proposition \ref{prop:mse-conc} holds for the more general case of sub-Gaussian r.v.s $\{X_1,\ldots,X_K\}$. However, in this case, the MSE $\cE_i$ is best in the class of linear estimators, and is not necessarily the minimum MSE estimator. 
 
\section{Uniform Sampling}
\label{sec:uniformsampling}
A simple approach towards identifying the best arm is to select each pair $(i,j) \in \S$ equal number of times, estimate the MSE errors ${\hat \cE}_{p}, \ p \in \cM$  and recommend the arm with the lowest MSE estimate to be optimal, i.e., the samples used for estimation are
$
   n_{ij}    = \frac{n}{ {K \choose 2} } = \frac{2n}{K(K-1)}, \ i \neq j.
$
\begin{theorem} \label{thm:unif_sampling}
For uniform sampling, the probability of error in identifying the optimal arm is 
 \begin{align*}
\P \left( {\hat A}_{n} \neq i^{*}\right) 
&  \leq    84K^2 \exp \left ( -      \frac{n l^2\Delta_{(1)}^2}{c K^7}  \right),
 \end{align*}
where 
$c$ is a universal constant. 
\end{theorem}
\begin{proof} Proof uses Proposition \ref{prop:mse-conc} along with an union bound and is provided in Section \ref{sec:unif_sampling_sketch}.
\end{proof}
 If the correlations between all pairs of arms and the variances of all the arms are similar, then the optimal strategy would involve sampling all pairs of arms an equal number of times.  However, when this is not the case, uniform sampling might be a strictly inferior strategy because it fails to gather more samples which can enable a better estimation of MSE of arms with MSE close to the optimal arm. We present below a strategy which tries to sequentially zone in on a reduced set of possible candidates for the optimal arm and then sample the pairs of arms involved in the MSE estimation of these arms approximately equal number of times to get a better probability of error in identifying the best arm. 
 
 \section{Successive Rejects}
 \label{sec:sr}

\begin{figure*}[ht]
\fbox{
\begin{minipage}{0.95\textwidth}
Set $A_1 = \S, B_1 = \{1,\ldots,K\}$ and for $k = 1, \ldots, K-2$,
\begin{align*}
 & n_{k} = \left \lceil{\frac{n-\binom{K}{2}}{C(K)\left(K+1-k\right)}}\right\rceil, 
 \text{ where } C ( K) =  \frac{K-1}{2} + \sum_{j=1}^{K-2}\frac{j}{K-j} \le K\log K.
\end{align*}

{\bfseries Phase $\bm{1}$}: Sample each pair $(i,j) \in A_1$ for $n_1$ number of times, estimate the MSE differences using \eqref{eq:mse-est-i}, remove the worst two arms from $B_1$ and the corresponding pair from $A_1$ to obtain $B_{2}$ and $A_{2}$ respectively.

{\bfseries Phase $\bm{k=2,\ldots,K-1}$:}
\begin{enumerate}
\item Pull each pair in $A_k$ $(n_{k}-n_{k-1})$ number of times. Estimate the MSEs using \eqref{eq:mse-est-i} and find the worst arm, say $a_{k+1}$, among the active arms in $B_{k}$.  
\item Set $B_{k+1} = B_{k } \setminus a_{k+1}$ and $A_{k+1} = A_{k} \setminus\{ (a_{k+1},a_{1}),  (a_{k+1},a_{2}), \ldots, (a_{k+1},a_{k})  \}$, where $B_{k}^c = \{a_{1}, \ldots, a_{k} \}$ is the set of arms that are out of contention by the end of phase $k-1$. 
\end{enumerate}

{\bfseries End of phase $\bm{K-1}$:}
Recommend the arm in $A_{K}$.
\end{minipage}
}
\caption{Successive rejects algorithm for correlated bandits.}
\label{fig:sr}
\end{figure*}

The successive rejects (SR) algorithm, which
pulls pairs of arms\footnote{With abuse of notation, $(a_i,a_j)$ is used to denote the (unordered) pair of arms $a_i,a_j$.} 
to identify the arm which minimizes MSE, operates over $K-2$ phases  as described in Figure \ref{fig:sr} .
The idea is to maintain a set of active arms and pairs of arms (for phase $k$, these are denoted by $A_k$ and $B_k$) and eliminate arms (and some of their corresponding pairs) that have high MSE. The elimination scheme employed in Figure \ref{fig:sr} departs significantly from the approach adopted in the classic SR algorithm for finding the arm with highest mean. To illustrate this, consider a setting with $5$ arms. If arms $4, \ 5$ are out of contention after phase $1$,  $A_2 = A_1 \setminus (4,5)$. In the second phase, all the pairs in $A_2$ are pulled  $(n_2-n_1)$ number of times. Now, if arm $3$ is out of contention at the end of this phase, the pairs $(3,4)$ and $(3,5)$ will be removed from $A_2$ and no longer be pulled in the later phases.   Notice that an approach that removes all tuples associated with arm $3$ is clearly sub-optimal, since tuples such as $(2,3)$, $(1,3)$ are necessary to obtain better estimates of $\hat \cE_1$ and $\hat \cE_2$. 

From the foregoing, the total number of samples used by SR is
 \begin{align*}
& \binom{K}{2} n_2 + \left[\binom{K}{2} - \binom{2}{2} \right](n_3 - n_2) + \ldots +\left[\binom{K}{2} - \binom{K-2}{2} \right](n_{K-1}-n_{K-2}) \\
& = \sum_{k=1}^{K-1} (k-1) n_k + (K-1) n_{K-1} < n, 
\end{align*}
where the final inequality follows by using the definition of $n_k$. 

Notice that a strategy that finds the worst arm according to empirical MSE estimates and discards all pairs associated with that arm is clearly sub-optimal, because samples from some of the discarded pairs of arms are essential to form estimate of MSE of arms which remain in contention. For e.g., in a $5$-armed bandit setting, suppose that we discard all pairs associated with arm $5$ in some round. This would impact the quality of MSE estimate of arm $1$, since the pair $(1,5)$ would be useful in training a better estimate of $\cE_1$ via $\rho_{15}$. 

Before presenting the main result that bounds the probability of error in identifying the best arm of the algorithm in Figure \ref{fig:sr}, we present the following problem complexities that
capture the hardness of the learning task at hand 
(i.e., the order of number of samples required to find the best arm with reasonable probability):
\begin{align}
 H_{2} = \underset{i} \max \frac{ i  }{  \Delta_{(i)}^2},  \text{ and }
  \ \overline H  =   \sum \limits_{i } \frac{  1 }{  \Delta_{i}^2 }. \label{eq:prob_complexities}
\end{align}
The quantities $H_2$ and $\overline H$,
  have a connotation similar to that in the classical bandit setup  and using arguments similar to those employed in \cite{audibert2010best}, it can be shown that
\begin{align*}  
H_2 \le {\overline H} \le \overline{\log} ( K) H_2,
\textrm{ where  }\overline{\log} ( K) =  \sum \limits_{i=2}^{K-2} \frac{1}{i}.
\end{align*}
Observe that the problem complexities depend both on the variances of the arms and the correlation between the arms through the gaps.
\begin{theorem}
\label{thm:succ-rejects}
 The probability of error in identifying the best arm of SR satisfies
 \begin{align*}
P( {\hat A}_{n} \neq i^{*} )   
  \leq
84K^3\exp \left ( - \frac{l^2}{c K^5}\frac{\left(n-\binom{K}{2}\right)}{C(K) H_2}\right),
\end{align*}
where  
$c$ is a universal constant, and $C(K)$ is defined in Figure \ref{fig:sr}.
\end{theorem}
\begin{proof}
See Section \ref{sec:sr-proof}.
\end{proof}
  From Theorem \ref{thm:unif_sampling}, it is apparent that an uniform sampling strategy would require $O(\frac{K^7}{\Delta_2})$ samples to achieve a certain accuracy, while our SR variant for correlated bandits would require $O(K^6 \bar H)$ number of samples.  
  SR scores over uniform sampling w.r.t. dependence on the number of arms $K$ because in our SR algorithm an increasing number of pairs of arms are removed from contention in successive phases.  More importantly, SR has better dependence on the underlying gaps when compared to uniform sampling. In problem instances where the gaps are uneven, SR finds the best arm much faster than uniform sampling.

\section{Lower Bound}
\label{sec:lb}

To obtain the lower bound, we consider a $K$-armed Gaussian bandit problem with the underlying joint probability distribution governed by the following covariance matrix:
\begin{align} \label{eq:lb_Cov_matrix}
 \Sigma  = \! \left[ \! \begin{array}{c c c c c c}
	1 & \rho  & \rho & \rho & \ldots & \rho \\
	\rho & 1 & \rho^2 & \rho^2 & \ldots  & \rho^2 \\
	\rho & \rho^2 & 1 & \rho^3 & \ldots & \rho^3\\
	\vdots & \vdots & \vdots & \vdots & \ddots & \vdots \\
	\rho  & \rho^2 & \rho^3 & \ldots & \rho^{K-1} & 1   
\end{array} \! \right] \! .
\end{align}
Observe that $\Sigma$ is a valid covariance matrix and is positive semi-definite.  
The MSEs corresponding to arms $1,\ldots,K$ are $\cE_1 = (K-1)(1-\rho^2), \ 
 \cE_2  = (1-\rho^2) + (K-2)(1-\rho^4) $ and more generally
\begin{align*} 
 \cE_i  = (i-1) - \sum_{i = 1}^{i-1} \rho^{2i} + (K-i)(1-\rho^{2i}) , \quad i = 1, \ldots,K.
\end{align*}
Hence, we have the following order on the MSEs: 
$\cE_1 \le \cE_2 \le \ldots \le \cE_K.$

 \subsubsection*{Problem transformations}
An approach in recent papers, cf. \citep{audibert2010best,kaufmann2015complexity}, for establishing lower bound for best-arm identification is to transform the bandit problem so that one of the sub-optimal arm is turned into an optimal one, while not affecting the rest of the arms. However, our setting involves correlated arms, with the correlation factors appearing in the mean-squared error objective and hence, one cannot make a sub-optimal arm optimal in a standalone fashion. We swap pairs of arms to interchange the MSE of a sub-optimal arm with that of the optimal arm and this introduces major deviations in the proof as compared classic $K$-armed case, as we shall soon see. We describe our problem transformations next.

We form $K-1$ transformations of the bandit problem formulated at the beginning of this section. For ``problem $m$,'' $m  =2, \ldots,K, $ arm $m$ is the best and for achieving this, we swap the first and $m$th rows in $\Sigma$.
Let $\cG$ be the pdf associated with the given problem as in \eqref{eq:lb_Cov_matrix}, and $\cG^{m}$ represent the pdf of the transformed bandit problem, where $m$ represents the $m$th transformation. 
Since we consider arms whose samples are i.i.d. in time, the 
joint distribution of $n$ samples is a product distribution of the underlying random variables $(\cG)^{\otimes n}$ and for the transformed problem by ${(\cG^m)}^{\otimes n}$. For compactness, we use $\mathbb{P}_{1} \triangleq \mathbb{P}_{ {(\cG)}^{\otimes n} }$,  $\mathbb{E}_{1} \triangleq \mathbb{E}_{ {(\cG)}^{\otimes n} } $ and $\mathbb{P}_{m} \triangleq \mathbb{P}_{ {(\cG^m)}^{\otimes n} } $, $\mathbb{E}_{m} \triangleq \mathbb{E}_{ {(\cG^m)}^{\otimes n} } $.

 \subsubsection*{Main result}
For any problem with
$\rho^2 \leq  UB_{\rho^2} \triangleq   1 - \frac{1}{\sqrt{K-2}} $,
we define $c_1 = \frac{1}{1-UB_{\rho^2}}  $ and $c_2 = \frac{\rho  }{1-UB_{\rho^2}}  $
and the min-max probability of error in identifying the optimal
arm is given by the theorem below.
\begin{theorem}\label{thm:lower-bd}
For any bandit strategy that returns the arm $\hat A_n$ after $n$ rounds, there exists a transformation of the covariance matrix such that the probability of error  on the transformed problem satisfies
\begin{align*}
&\max_{1 \leq m \leq K}\P_{m}(\hat A_n \neq m) \geq  \frac{1}{6}\exp\Big(- \frac{6 n K}{H_{lb}} -{K \choose 2} n\tilde \epsilon_{n}\Big), 
\end{align*}
where  
$H_{lb} = \sum \limits_{ i \neq 1} \frac{1}{ \Delta_i}$ is the problem complexity term,\\
 $\tilde \epsilon_n =  {\tilde c} u \max \Big\{ \frac{8}{n}\log  {12K(K-1)n}  , $ $ \sqrt{\frac{8}{n}\log  {12K(K-1)n} } \Big \}$, and ${\tilde c} = \max\left(3 c_1, 48 c_2\right)$. 
\end{theorem}
\begin{proof} See Section \ref{sec:appendix-lower-bd}. \end{proof}

Note the gap between the upper and lower bounds on the probability of error in Theorems \ref{thm:succ-rejects} and \ref{thm:lower-bd}. The problem complexity term in the upper bound involved the square of the gaps, whereas the lower bound involves just the gaps. 
We believe the upper bound for SR algorithm is optimal in terms of gap dependence and it would be interesting future work to establish a lower bound that involves squares of the gaps. In the lower bound proof, the Kullback-Leibler divergence terms for the transformed problems were bounded above by the gaps (for e.g., see \eqref{eq:kl-ub-in-lb} in Section \ref{sec:appendix-lower-bd}), leading to an overall lower bound with complexity $H_{lb}$. Nevertheless, the current proof is challenging owing to (i) pairs of arms being pulled in each round; (ii) the covariance matrix in \eqref{eq:lb_Cov_matrix} is non-trivial and its problem transformations are novel and finally, (iii) arriving at the bound for the aforementioned KL-divergence terms requires non-trivial algebraic effort.


\section{Proofs}
\label{sec:proofs}
\input{proofs}

\section{Numerical Experiments}
\label{sec:expts}
We show a few simple experiments here to illustrate our theoretical analysis. Since, this line of work is new, we compare our successive rejects type algorithm and uniform sampling which is optimal in some settings. We show three experiments,  in which all the arms are jointly Gaussian having mean zero and unit variance. Each experiment can be seen to consist of two clusters of arms with the arms in each cluster being independent of the arms in the other cluster. Arm 1, in the first cluster, is optimal in all the three experiments and the  arms in the second cluster are typically less correlated among themselves than the arms in the first cluster. 

In a setting with 35 arms, we employ the following covariance matrices for the three experimental setups: 
\begin{align} 
\! \Sigma_1 \! = \! \left[ \begin{array}{c  c}
           {\bf M}_{1}  &  {\bf 0}  \\	
	     {\bf 0}    & {\bf I}_{25 \times 25}	
\end{array} \! \right] \! ,\ 
 \Sigma_2 \! = \! \left[ \! \begin{array}{c  c}
           {\bf M}_{1}  &  {\bf 0}  \\	
	    {\bf 0} &  {\bf Tr}_{31 \times 31}	
\end{array} \! \right] \! ,\label{eq:cov1}
\\[0.5ex]
\Sigma_3  = \! \left[ \! \begin{array}{c c c c c}
          1 & 0.5     & 0.45  & 0.5  &  {\bf 0}   \\
	0.5 & 1   & 0.45  & 0.4   &  {\bf 0}  \\
	0.45 & 0.45   & 1  & 0.4   &  {\bf 0}  \\
	0.5 & 0.4   & 0.4  & 1   &  {\bf 0} \\	
	& & {\bf 0} & & {\bf I}_{30 \times 30}
\end{array} \! \right] \!,\label{eq:cov3}
\end{align}
 where ${\bf M}_{1} = [ 1 \ 0.9 \ 0.9 \ 0.9 ;
 \ \ \	0.9 \ 1   \ 0.85  \ 0.85 ;$\\
$ \ \ \	0.9 \ 0.85 \ 1  \ 0.85  ;$
 $\ \ \ 	0.9 \ 0.85  \ 0.85  \ 1    ]$, 
and ${\bf Tr}_{K-5 \times K-5}$ is a tridiagonal matrix with ones along the main diagonal, $0.2$ in the diagonals above and below and zeros elsewhere.
Notice that $\Sigma_i$ is a block diagonal matrix for each $i=1,2,3$ and hence, its eigenvalues are union of the nonzero diagonal submatrices. It is easy to verify that the individual blocks, i.e., $M_1$ and ${\bf Tr}_{K-5 \times K-5}$, are positive semi-definite and hence, so is $\Sigma_i$ for each $i$.

In Example 1 corresponding to covariance matrix $\Sigma_1$, 
arms in the first cluster are highly correlated amongst themselves,
and arms in the second cluster are independent of all the arms.
On the other hand, in Example 2 corresponding to covariance matrix $\Sigma_2$,  arms in the first cluster are highly correlated among themselves
and the arms in the second cluster are weakly correlated amongst themselves.
Finally, in Example 3 corresponding to covariance matrix $\Sigma_3$, arms in the first cluster are weakly correlated among themselves
and arms in the second cluster are independent of all the arms.
In all the three examples, multiple arms in the first cluster have MSE close to that of the optimal arm. 
Clearly, more samples of the pairs of arms corresponding to the first cluster of arms are required to identify the optimal arm accurately.
As number of arms $K$ increases, the proportion of samples used for pairs corresponding to the clearly sub-optimal arms increases at a faster rate for uniform sampling algorithm as compared to SR. 

We conduct our experiments with the number of samples equaling $\cong \frac{{\overline H}}{32^2}$ for each experiment. Figure \ref{fig:errors} compares the probability of error for the three settings with covariance matrices given in \eqref{eq:cov1}--\eqref{eq:cov3}.  In all three settings, SR recommends the optimal arm with higher probability, and this is because SR algorithm rejects the sub-optimal
arms in the beginning phases using fewer samples and allocates more samples to the first cluster to 
distinguish between the arms in this cluster.

\begin{figure}
\centering
      \tabl{c}{\scalebox{0.85}{\begin{tikzpicture}
      \begin{axis}[
      ybar={2pt},
       legend pos=north east,
      legend image code/.code={\path[fill=white,white] (-2mm,-2mm) rectangle
      (-3mm,2mm); \path[fill=white,white] (-2mm,-2mm) rectangle (2mm,-3mm); \draw
      (-2mm,-2mm) rectangle (2mm,2mm);},
      ylabel={\bf Probability of error},
      xlabel={},
      symbolic x coords={0, 1, 2, 3, 4 },
      xmin={0},
      xmax={4},
      xtick=data,
      ytick align=outside,
      xticklabels={{\bf Experiment 1,\bf Experiment 2,\bf Experiment 3 }},
      xticklabel style={align=center},
      bar width=10pt,
      grid,
      grid style={gray!30},
      width=12cm,
      height=7cm,
      ]

      \addplot   coordinates {  (1,0.46)  (2,0.53 ) (3,0.13) }; 
      \addlegendentry{Uniform}
      \addplot coordinates { (1,0.35)   (2,0.43 ) (3,0.08) }; 
      \addlegendentry{SR}

      \end{axis}
      \end{tikzpicture}}\\[1ex]}
      \caption{Probability of error for uniform sampling and SR algorithms on three different problems, with covariance matrices $\Sigma_1,\ldots,\Sigma_3$, respectively.  The results are averages from 200 independent replications.}
      \label{fig:errors}
\end{figure}
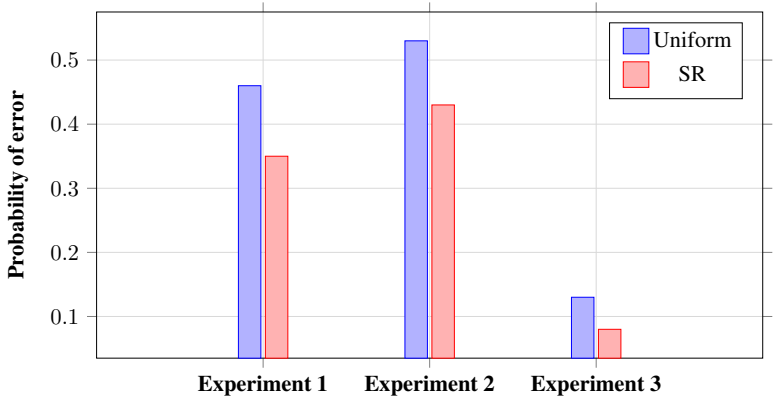

\section{Conclusions}
\label{sec:conclusions}
We presented a new formulation of the $K$-armed bandit problem where the goal, using the MSE criterion, is to find an arm that best captures information about all arms. Both estimation of MSE for individual arms, and exploration to find the best arm in a correlated bandit are challenging. We proposed an MSE estimator that uses samples from the distribution underlying any pair of arms, and  showed that our estimator concentrates. We adapted the SR algorithm to successively eliminate arm pairs, and proved a bound on the probability of error in identifying the best arm. We also derived a lower bound for the correlated bandit problem.

\subsubsection*{Acknowledgments}
The authors would like to thank Akshay Reddy for inputs on the phase lengths in the SR algorithm.
The first author would like to thank Prof. Prakash Narayan for guidance on the sampling rate distortion problem, which led to the current work. 
This work was supported by the U.S. National Science Foundation under the Grant CCF-1319799.

\bibliographystyle{plainnat}
\bibliography{refs}

\end{document}

%% file: proofs.tex
\subsection{Problem complexities} \label{sec:prob-complexities}

The relation between the different problem complexities, defined in the Section \ref{sec:model}, is shown below.
\begin{align*}
 {\overline H} & = \sum \limits_{i} \frac{ 1   }{ \left (  {\Delta_{i}  }  \right )^2} = \sum \limits_{i} \frac{1}{i} \left ( i \frac{ 1  }{ \left (  {\Delta_{i} }  \right )^2} \right )  \leq {\overline \log} (K) \ \underset{ i } \max \     \frac{i    }{  \left (  {\Delta_{(i)}  }  \right )^2} =  {\overline \log} (K) H_{2} \\
 {\overline H} & = \sum \limits_{i} \frac{ 1   }{ \left (  {\Delta_{i} }  \right )^2}\geq  \frac{ {\tilde i}  }{ \left (  {\Delta_{({\tilde i})} } \right )^2} = H_{2}, \ {\tilde i}  \text{ is the optimizer of $H_{2}$, and finally} \\
{\overline H} & \geq   \frac{1}{Ku}  H_{lb}, \ \text{since }  \Delta_{i} \leq K u, \ i = 1, \ldots, K.
\end{align*}

\subsection{Proof of Proposition \ref{prop:mse-conc}}
\label{sec:proofs-mse-conc}
\input{proof_conc}

\subsection{Proof of Theorem \ref{thm:unif_sampling}}
\label{sec:unif_sampling_sketch}
\begin{proof}
Without loss of generality, we assume that $1$ is the optimal arm. 
 In uniform sampling, each pair of arms is sampled $\frac{n}{{K \choose 2}}$ times. Using this, along with Proposition \ref{prop:mse-conc} and union bound we obtain
 \begin{align*}
\P ( {\hat A}_{n} \neq 1)  &\leq \sum \limits_{i=2}^{K} \P   (   \hat\cE_{1} \geq \hat\cE_i   ) \\
&\le \sum \limits_{i=2}^{K} \P   \left(   \hat\cE_{1} - \cE_1 \ge \frac{\Delta_i}{2}\right) + \P   \left(   \hat\cE_{i} - \cE_i \le -\frac{\Delta_i}{2}\right) \\
& \le  84K \sum \limits_{i=2}^{K} \exp \left ( - \frac{n}{{K \choose 2}} \frac{ l^2}{31104 K^5}  \Delta_i^2 \right)  \quad \textrm{(Using Proposition \ref{prop:mse-conc})}\\
& \le   84K^2 \exp \left ( - \frac{n}{{K \choose 2}}     \frac{l^2\Delta_{(1)}^2}{31104 K^5}  \right).
\end{align*}
\end{proof}
\subsection{Proof of Theorem \ref{thm:succ-rejects}}
\label{sec:sr-proof}

\begin{proof}
Without loss of generality, we assume that $1$ is the optimal arm. 
 If the optimal arm is eliminated from $B_{k}$ at the end of phase $k $, then 
 \begin{align*}
  & \hat \cE_{1, n_{k}} \geq \underset{ i \in \{ K, K-1, \ldots, K -k  \} } \min \hat \cE_{(i), n_{k}},
 \end{align*}
 where $\cE_{i, n_{k}} $ is the MSE estimate of arm $i$ using $n_{k}$ samples for all pairs of arms involved in the estimation as indicated in \eqref{eq:mse-est-i}.
Let $ E_k$ be the event that optimal arm is eliminated at the end of phase  $k $,
then $
 E_k \subseteq \bigcup_{i = K-k}^{K} \{ \hat  \cE_{1, n_{k}} \geq \hat \cE_{(i), n_{k}} \}. $
The probability of error of this algorithm, using an union bound is \vspace*{-0.2cm}
\begin{align*}
  P( {\hat A}_{n} \neq 1 ) &\leq \sum \limits_{k=1}^{K-1} \P \left ( E_{k} \right )  \leq \sum \limits_{k=1}^{K-1} \sum \limits_{i=K+1-k}^{K} \P \left ( \hat \cE_{1, n_{k}} - \cE_{1}  + \cE_{(i)} - \hat \cE_{(i), n_{k}}  \geq \Delta_{(i)} \right ) \\
  & \leq \sum \limits_{k=1}^{K-1} \sum \limits_{i=K+1-k}^{K} \P \left ( \hat \cE_{1, n_{k}} - \cE_{1}   \geq \frac{\Delta_{(i)}}{2} \right )
  + \P \left ( \hat \cE_{(i), n_{k}} -\cE_{(i)} \leq -\frac{\Delta_{(i)}}{2} \right )\\
  & \le  84K \sum \limits_{k=1}^{K-1} \sum \limits_{i=K+1-k}^{K} \exp \left ( - n_k      \frac{l^2 \Delta_{(i)}^2}{c K^5} \right) \qquad \textrm{(using Proposition \ref{prop:mse-conc})}  \\
  & \le 84K \sum \limits_{k=1}^{K-1} k\exp \left ( - n_k    \frac{l^2\Delta_{(K+1-k)}^2}{c K^5}\right  ) \quad \textrm{(since } \Delta_{(K+1-k)} \le \Delta_{(i)}, i=K+1-k,K).
\end{align*}
Now, 
\begin{align*}
   {n_{k}}\Delta_{(K+1-k)}^2   \geq & \  \frac{(n-\binom{K}{2})}{C(K)(K+1-k)\Delta^{-2}_{(K+1-k)}}  
   \geq \frac{(n-\binom{K}{2})}{C(K) H_2},
\end{align*}
where $H_2$ is as defined in \eqref{eq:prob_complexities}.
Hence, we obtain
\begin{align*}
  P( {\hat A}_{n} \neq 1 ) &\leq  84K^3\exp \left ( - \frac{l^2}{c K^5}\frac{(n-\binom{K}{2})}{C(K) H_2}\right).
  \end{align*}
\end{proof}


\subsection{Proof of Theorem \ref{thm:lower-bd}}
\label{sec:appendix-lower-bd}

\begin{proof}
We follow the technique from \cite{carpentier2016tight} for establishing the lower bound. However, our setting involves correlated arms, with the correlation factors appearing in the mean-squared error objective. More importantly, instead of the translation of the mean of sub-optimal arms performed in \cite{carpentier2016tight}, we swap the variances of a pair of arms to interchange the MSE of a sub-optimal arm with that of the optimal arm. 


Before proceeding with the main proof, we state a well-known result for the KL-divergence between multivariate normal distributions that we will use several times.
\begin{lemma}
\label{lemma:klNormalGeneral}
Let ${\cN}_{0}, \ {\cN}_{1} $ be two $k$-dimensional normal distributions with zero-mean and covariance matrix $A_{0}, \ A_{1}$, 
respectively.
\begin{align*}
 KL( {\cal N}_{0} || {\cal N}_{1} ) = \frac{1}{2} \left (  \text{Tr} \big (  A_{1}^{-1} A_{0} \big ) -k + ln \left ( \frac{\text{det}( A_1)}{\text{det}( A_0)} \right ) \right ).
\end{align*}
\end{lemma}

Consider problem $m$ with underlying covariance matrix $\Sigma_m$. 
For $(i,j) \in  \S$, let $\nu_{i}\nu_j$ and $\nu_{i}' \nu_j'$ denote bivariate normal distributions with variance and correlations specified by $\Sigma_1$ and $\Sigma_i$, respectively. 
Let $\kl_{ij}^m\triangleq KL( \nu_i \nu_j || \nu_i' \nu_j' )$ denote the Kullback-Leibler divergence between $\nu_{i}\nu_j$ and $\nu_{i}' \nu_j'$, where the latter distributions are derived from $\cG_m$.

Notice that, for $\kl_{12} \triangleq ( \nu_1 \nu_2 || \nu_1' \nu_2' )$ is zero in each of the $K-1$ transformations, since the underlying covariance matrices corresponding to $\nu_1 \nu_2$ as well as $\nu_1' \nu_2'$ is 
$\left[ \begin{array}{c c }
	1 & \rho  \\
	\rho & 1    
\end{array} \! \right]$. 

Consider the second problem transformation, where the first and second rows of $\Sigma$ are interchanged to make arm $2$ optimal. Under this transformation, for $KL_{13}$, the matrices $A_0$ and $A_1$ for application of Lemma \ref{lemma:klNormalGeneral} are 
$\left[ \begin{array}{c c }
	1 & \rho  \\
	\rho & 1    
\end{array} \! \right]$ and 
$\left[ \begin{array}{c c }
	1 & \rho^2  \\
	\rho^2 & 1    
\end{array} \! \right]$, respectively.
Thus, for $j=3,\ldots,K$, we have
\begin{align*}
\kl^2_{1j} & = \frac{1}{2} \left( \frac{2(1-\rho^3)}{1-\rho^4}  - 2  + ln \frac{1-\rho^4}{1-\rho^2} \right)\\
& \le \frac{1}{2} \left( \frac{2(\rho^4-\rho^3)}{1-\rho^4}  + \frac{1-\rho^4}{1-\rho^2} -1 \right)\\
& = \frac{\rho^2}{2} \frac{(1-2\rho + 2 \rho^2 - \rho^4)}{1-\rho^4} \leq \frac{\rho^2}{2} \\
& = \frac{\rho^2 (1-\rho^2)(K-2) }{2 (1-\rho^2)(K-2)} \\
 & \le \Delta_2. \stepcounter{equation}\tag{\theequation}\label{eq:kl-ub-in-lb}
\end{align*}
if $\rho^2 \le \frac{2K-5}{2K-4}$.

Along similar lines, for $j=3,\ldots,K$, we have
\begin{align*}
\kl^2_{2j} & = \frac{1}{2} \left( \frac{2(1-\rho^3)}{1-\rho^2}  - 2  + ln \frac{1-\rho^2}{1-\rho^4} \right)\\
& \le \frac{1}{2} \left( \frac{2(\rho^2-\rho^3)}{1-\rho^2}  + \frac{1-\rho^2}{1-\rho^4} -1 \right)\\
& \leq \frac{\rho^2}{2} \frac{(1-\rho)}{1+\rho} \\
& \le \Delta_2,
\end{align*}
if $ \rho^2 \le \frac{2K-5}{2K-4}$.
The other KL divergences, i.e., $ \kl_{12}^2, \  \kl^2_{ij}$, $i,j \notin \{1,2\}$ are zero.

Using arguments analogous to those above, the bounds for $\kl_{ij}$ can be derived for problem $m$,  $m = 3, \ldots, K$.
In general, for problem $m$, 
\begin{align}
 \kl_{1m}^{m} = 0 , \ \kl_{ij}^{m} = 0 , \ i,j \notin \{1,m \}.
\end{align}
When $ 1 < j < m$, $\kl_{1j}^{m}$ is the KL between
${\cal N} \left (0, \rho I_{2 \times 2} \right )$
and ${\cal N} \left (0, \rho^{j} I_{2 \times 2} \right )$
\begin{align*}
 \kl_{1j}^{m} \leq \frac{\rho^2}{2} \frac{1 - \rho^{2(j-1)}}{1- \rho^2}
\end{align*}
and for $ m < j < K$,  $\kl_{1j}^{m}$ is the KL between
${\cal N} \left (0, \rho I_{2 \times 2} \right )$
and ${\cal N} \left (0, \rho^{m} I_{2 \times 2} \right )$
\begin{align*}
 \kl_{1j}^{m} \leq  \frac{\rho^2}{2} \frac{1 - \rho^{2(m-1)}}{1- \rho^2}.
\end{align*}
and hence for $j \neq m$
\begin{align*}
\quad  \kl_{1j}^{m} \leq \Delta_2 \leq \Delta_m, \ \text{ if } \rho^2 \le   1- \frac{1}{\sqrt{2K-4}}  .
\end{align*}
Similarly for $\kl_{mj}^{m}$, 
when $ 1 < j < m$, $\kl_{mj}^{m}$ is the KL between
${\cal N} \left (0, \rho^j I_{2 \times 2} \right )$
and ${\cal N} \left (0, \rho I_{2 \times 2} \right )$
\begin{align*}
 \kl_{1j}^{m} \leq  {\rho^2}  \frac{1 - \rho^{ (j-1)}}{1- \rho^2}
\end{align*}
and for $ m < j < K$,  $\kl_{mj}^{m}$ is the KL between
${\cal N} \left (0, \rho^m I_{2 \times 2} \right )$
and ${\cal N} \left (0, \rho I_{2 \times 2} \right )$
\begin{align*}
 \kl_{mj}^{m} \leq   {\rho^2}  \frac{1 - \rho^{ (m-1)}}{1- \rho^2}.
\end{align*}
and hence for $j \neq m$
\begin{align*}
\quad  \kl_{mj}^{m} \leq \Delta_2 \leq \Delta_m, \ \text{ if } \rho^2 \le   1- \frac{1}{\sqrt{ K-2}}  .
\end{align*}

\subsubsection*{Change of measure}

For $(i,j) \in \S$, let $N_{ij}$ denote the number of samples obtained from the joint distribution of $(X_i,X_j)$. 
Let 
$$n_{ij} = \E_{1} N_{ij}, \quad (i,j) \in \S.$$
Observe that  $\sum\limits_{(i,j) \in \S} N_{ij} = \sum\limits_{(i,j) \in \S} n_{ij} = n$.

Notice that the problem transformations impact the distribution of each arm and hence, we cannot employ a change of measure identity similar to \cite{audibert2010best}. 
Instead, we factor in the KL-divergences $\kl_{ij}, \ \forall (i,j) \in \S$ and derive a change of measure identity as follows:  for any measurable event $\cE$ based on the samples, 
\begin{align}
\P_{m}(\cE) &= \E_{\nu_1 \nu_2 \nu_3}\left[\mathbf{1}\{\cE\} \prod\limits_{(i,j) \in \S} \prod\limits_{s \in {\cal T}_{ij}(n)}\frac{ d \nu_{i}' \nu_{j}' }{d \nu_{i} \nu_{j}}(X_{i,s}, X_{j,s})\right]\nonumber\\
& = \E_{\nu_1 \nu_2 \nu_3}\left[\mathbf{1}\{\cE\}\exp\left(\sum\limits_{(i,j)\in\S} -N_{ij}\ekl_{ij,N_{ij}}\right)\right] \nonumber\\
& = \E_{1}\left[\mathbf{1}\{\cE\}\exp\left(\sum\limits_{(i,j)\in\S} -N_{ij}\ekl_{ij,N_{ij}}\right)\right], \label{eq:change-measure}
\end{align}
where ${\cal T}_{ij}(n)$ is the set of time instants when the algorithm pulled the tuple $(i,j)$. 
For $(i,j) \in \S, 1\leq t\leq n$, let 
\begin{align*}
  \ekl_{ij,t}  &= \frac{1}{t} \sum \limits_{s=1}^{t} \log \frac{ d \nu_{i} d \nu_{j} }{d \nu_{i}' d \nu_{j}'}(X_{i,s}, X_{j,s}) \\
 &= \frac{1}{2t} \sum \limits_{s=1}^{t} [X_{i,s} \ X_{j,s}] ( \Sigma_{(i,j),1}^{-1} - \Sigma_{(i,j),m}^{-1}) \left[\begin{array}{c}
	X_{i, s} \\
	X_{j, s}  
\end{array}\right]  +\frac{1}{2} ( \log | \Sigma_{(i,j),m}| -  \log | \Sigma_{(i,j),1} | ),
\end{align*}
where $(X_{i,s},X_{i,s}) $ are i.i.d. $\sim \cG $ for all $s \leq t.$ Here, $\Sigma_{(i,j),m}$ is the covariance matrix of $X_{i}, X_{j}$ for problem $m$, and is a submatrix of $\Sigma_{m}.$

Notice that the change of measure is from $\nu_1',\ldots,\nu_K'$ to $\nu_1,\ldots,\nu_K$, where the transformed distributions $\nu_i'$ are governed by the covariance matrix $\Sigma_m$. 
Thus, the distributions of $\nu_{i}', \nu_{j}'$ in the definition of $\ekl_{ij,t}$ are to be interpreted as coming from the appropriate bivariate distribution in $\cG_m$.

\subsubsection*{Concentration of empirical divergences}
The following lemma shows that the empirical divergences concentrate, for all bandits $i=1,\ldots,K$.
\begin{lemma}\label{xi}
Consider the following event:
$$
\xi = \Big\{\forall (i,j) \in \S, 1\leq t\leq n, \ekl_{(ij),t}  - \kl_{ij}\leq \tilde \epsilon_t \Big\},
$$
where $\tilde \epsilon_t$ is as defined in the theorem statement.
Then, for $m=1,\ldots,K$,
$$\P_{m}(\xi) \geq 5/6.$$

\end{lemma}
\begin{proof}
By definition of $\ekl$, $\kl$ and $\Sigma_2$, we have 
\begin{align*}
\E_{1} \ekl_{(i,j),t}& = 0 , \text{ if }  i,j \notin \{1,m\} \text { or }  i = 1, j = m, \\
\E_{1} \ekl_{(i,j),t}& = \kl_{ij}^m , \text{ either } i=1 \text{ or } j = m .
\end{align*}

Notice that, more generally when the variance of all $X_{i}$s is $1$  
\begin{align*}
 \Sigma_{(i,j),1}^{-1} &= \frac{ 1 }{ (1-\rho^{2 min(i,j)})}
 \left[\begin{array}{cc}
      1 & -\rho^{min(i,j)}   \\
	-\rho^{min(i,j)}   & 1 
\end{array}\right], \  1 \leq i \neq j < m
 \\
 \Sigma_{(1,j),m}^{-1} &= \frac{ 1 }{ (1-\rho^{2j})}
 \left[\begin{array}{cc}
      1 & -\rho^j   \\
	-\rho^j   & 1 
\end{array}\right], \  2 \leq j < m,  
\\
 \Sigma_{(1,j),m}^{-1} &= \frac{ 1 }{ (1-\rho^{2m})}
 \left[\begin{array}{cc}
      1 & -\rho^m \\
	-\rho^m & 1 
\end{array}\right], \ j > m, \textrm{ and }\\
 \Sigma_{(m,j),m}^{-1} &= \frac{ 1 }{ (1-\rho^{2 })}
 \left[\begin{array}{cc}
      1 & -\rho    \\
	-\rho    & 1 
\end{array}\right], \  2 \leq j < m
 \end{align*}
Then,
for $a_{1}$ and $a_{2}$  denote the first and second element the first row of
the matrix $\Sigma_{(i,j),1}^{-1} - \Sigma_{(i,j),m}^{-1}, $ respectively.
Then, $a_1 \leq \left(\frac{ 1 }{1-\rho^2} - \frac{ 1 }{1-\rho^{2m}}\right)$,
$a_2 \geq -\frac{\rho  }{1-\rho^2}  $. Letting 
$c_1 = \frac{ 1 }{1-UB_{\rho^2}}  $, and
$c_2 = \frac{\rho  }{1-UB_{\rho^2}}  $, we obtain
\begin{align*}
& \P \left( \ekl_{(i,j),t} - \kl_{i,j}^m > \epsilon\right) \\
& \le \P\left(   \frac{1}{t} \sum \limits_{s=1}^{t} a_1 X_{i,s}^2 - a_1   > \frac{\epsilon}{3}\right) + \P\left(   \frac{1}{t} \sum \limits_{s=1}^{t} a_2 X_{i,s}X_{j,s} - a_2 \rho_{ij}   > \frac{\epsilon}{6}\right)
+ \P\left(   \frac{1}{t} \sum \limits_{s=1}^{t} a_3 X_{j,s}^2 - a_3    > \frac{\epsilon}{3}\right) \\
& \le \P\left(   \frac{1}{t} \sum \limits_{s=1}^{t} a_1 X_{i,s}^2 - a_1    > \frac{\epsilon}{3}\right) + \P \left (   \frac{1}{t} \sum \limits_{s=1}^{t} a_{2} \big ( (X_{i,s} + X_{j,s})^2 -  ( 1 + 2 \rho^{\min(i,j)}  ) \big ) > \frac{\epsilon}{12} \right ) \\
 & \quad +   \P \left (   \frac{1}{t} \sum \limits_{s=1}^{t} a_{2} \big ( (X_{i,s} - X_{j,s})^2 -  (2 - 2 \rho^{\min(i,j)} ) \big ) > \frac{\epsilon}{12} \right ) + \P\left(   \frac{1}{t} \sum \limits_{s=1}^{t} a_3 X_{j,s}^2 - a_3    > \frac{\epsilon}{3}\right)\\
& \le 2 \concsubexpTwoterms{t}{(\frac{\epsilon}{3 c_1 })}{u} + 2 \concsubexpTwoterms{t}{(\frac{\epsilon}{12 c_2})}{(4u)}  \\
& \le 4\exp \left( - \frac{t}{8}\min\left(\frac{\epsilon}{\hat c u}, \frac{\epsilon^2}{\hat c^2 u^2}\right)  \right), \text{ where } \hat c = \max\left(3 c_1, 48 c_2\right),
\end{align*}
using Lemma \ref{lemma:subexp-conc}, together with $\rho^2 \leq 1- \frac{1}{\sqrt{ K-2}}$.

 Plugging in $\tilde \epsilon$ in the last inequality above, we obtain
\begin{align*}
\P \left( \ekl_{(i,j),t} - \kl_{i,j}^m > \tilde\epsilon\right) \le  \frac{1}{3K(K-1)n}
\end{align*}
The main claim follows by using a union bound in conjunction with the equation above.
\end{proof}

Consider now the event 
\[\cE = \{\hat A_n = 1\} \cap \{\xi\} \cap \{N_{ij} \leq 6 n_{ij}, i=1,\ldots,K\}.\]
For $m=2,\ldots,K$, we have the following from the change of measure identity in \eqref{eq:change-measure}:
\begin{align}
\P_{m}(\cE)  &=\mathbb{E}_{1}\Big[\mathbf{1}\{\cE\}\exp\big(-\sum\limits_{(i,j)\in\S} N_{ij} \ekl_{ij,N_{ij}}\big)\Big]\nonumber\\
& \geq \mathbb{E}_{1}\Big[\mathbf{1}\{\mathcal E\}\exp\Big(-\sum\limits_{(i,j)\in\S} \left(N_{ij}\kl^m_{ij} + N_{ij}\tilde \epsilon_{N_{ij}}\right)\Big)\Big] \label{eq:chicall}\\
& \geq \mathbb{E}_{1}\Big[\mathbf{1}\{\mathcal E\}\exp\Big(-\sum\limits_{(i,j)\in\S} 6n_{ij}\kl^m_{ij} - {K \choose 2} n\tilde \epsilon_{n}\Big)\Big]\label{eq:chicall2}\\
& \geq \exp\Big(-\sum\limits_{(i,j)\in\S}6n_{ij}\kl^m_{ij} -{K \choose 2} n\tilde \epsilon_{n}\Big) \P_{1}(\mathcal E).\label{eq:event1}
\end{align}
The inequality in \eqref{eq:chicall} follows from Hoeffding's inequality, while the inequality in \eqref{eq:chicall2} follows from the fact that
on $\mathcal E$, $N_{ij} \leq 6n_{ij},\ \forall i$. 

\subsubsection*{Clinching argument}
It suffices to consider algorithms that satisfy $\mathbb E_{1}(\hat A_n \neq 1) \leq 1/2$. 
Then, $\mathbb P_{1} (N_{ij} \geq 6 n_{ij}) \leq \frac{\mathbb E_1 N_{ij}}{6n_{ij}} = 1/6$ by Markov inequality. 
Thus, we have 
$$\mathbb P_{1}(\mathcal E) \geq 1 - (1/6+1/2 +1/6) = 1/6.$$  
Using the equation above and \eqref{eq:event1}, we obtain
\begin{align}
\P_{m}(\hat A_n \neq m)  &\geq \frac{1}{6}\exp\Big(-\sum\limits_{(i,j)\in\S}6n_{ij}\kl^m_{ij} -{K \choose 2} n\tilde \epsilon_{n}\Big)\nonumber\\
&\geq \frac{1}{6}\exp\Big(-\sum\limits_{(i,j)\in \S}6  n_{ij}\Delta_m -{K \choose 2} n\tilde \epsilon_{n}\Big),\label{eq:final}
\end{align}
where the  final inequality follows from the bounds on $\kl^m_{ij}$ derived earlier.

Let $l_{i} = n_{12} + n_{13} + \ldots + n_{1K} + \sum \limits_{j \neq 1,i} n_{ij}$. Then, there exists an $i$ such that
\begin{align}
l_{i }  \leq  \frac{n K}{H_{lb} \Delta_i}, \label{eq:lispecial}
\end{align}
where $H_{lb}$ is as defined in the theorem statement.
 For, if not
\begin{align*}
   \sum \limits_{ i \neq 1} l_{i} &= \sum \limits_{ i \neq 1}  \left (  n_{12} + n_{13} + \ldots + n_{1K} + \sum \limits_{j \neq 1,i} n_{ij}  \right ) \\
 &= (K-1)( n_{12} + n_{13} + \ldots + n_{1K}  )  + 2 \sum \limits_{i  \neq j \  i,j \neq 1} n_{ij} \\
  &  = (K-3) ( n_{12} + n_{13} + \ldots + n_{1K}  )  + 2 \sum \limits_{i  \neq j } n_{ij}
\end{align*}
\begin{align*}
&\Rightarrow (K-3)(n_{12} + n_{13} + \ldots n_{1K})  + 2n =  \sum \limits_{i} l_{i}  >  \sum \limits_{i \neq 1} \frac{n K}{H_{lb} \Delta_i},
\end{align*}
which is a contradiction, since $n_{12} + n_{13} + \ldots n_{1K} \leq n.$

For $i$ that satisfies the condition in \eqref{eq:lispecial}, from \eqref{eq:final} we have 
\begin{align*}
&\P_{m}(\hat A_n \neq m)  \geq \frac{1}{6}\exp\Big(- \frac{6 nK}{H_{lb}} -{K \choose 2} n\tilde \epsilon_{n}\Big).
\end{align*}
The main claim follows.
\end{proof}

%% file: proof_conc.tex
The MSE estimate in \eqref{eq:mse-est-i} involves sample variances and sample correlation coefficients, and hence, MSE concentration requires each of these quantities to concentrate. While one can use Bernstein's inequality  for handling sample variance, a finite sample concentration bound for sample correlation coefficient does not exist, to the best of our knowledge. We begin by stating 
two well-known bounds concerning concentration of sample variance and standard deviation. Subsequently, we provide a concentration result for sample correlation coefficient, and prove Proposition \ref{prop:mse-conc} using the aforementioned bounds.
\begin{lemma}\textbf{\textit{(Concentration of sample variance)}} 
\label{lemma:subexp-conc}
Let $X_{i}$,\  $i=1,\ldots,n,$ be independent sub-Gaussian r.v.s with common parameter $\sigma$.
Let
$\hat \sigma^2_n =  \frac{1}{n} \sum \limits_{i=1}^{n} X_{i}^{2}$. Then, we have the following bound for any $\epsilon \ge 0$:
\begin{align*}
 & P \left (  \hat\sigma^2_n >  \sigma^2 + \epsilon   \right)  \leq 
 \exp \left (  - \frac{n}{8} \min \left ( \frac{\epsilon^2}{\sigma^4}, \frac{\epsilon }{\sigma^2}\right ) \right ), \textrm { and }
  P \left ( \hat\sigma^2_n <  \sigma^2 - \epsilon   \right)  \leq 
 \exp \left (  - \frac{n}{8} \min \left ( \frac{\epsilon^2}{\sigma^4}, \frac{\epsilon }{\sigma^2}\right ) \right ).
\end{align*} 
\end{lemma}
\begin{proof}
By definition, it follows that the square of a sub-Gaussian r.v. is sub-exponential. The main claim now follows from the concentration bound for sub-exponential r.v.s in Proposition 2.2 of \citep{wainwright2019high}. 
\end{proof}

\begin{lemma}\textbf{\textit{(Concentration of sample standard deviation)}} 
\label{lemma:sample-stddev-conc-bd}
Under conditions of Lemma \ref{lemma:subexp-conc}, letting
$\hat \sigma_n =  \sqrt{\frac{1}{n} \sum \limits_{i=1}^{n} X_{i}^{2}}$, we have
\begin{align*}
 & P \left ( \hat\sigma_n >  \sigma + \epsilon   \right)  \leq 
 \exp \left (  - \frac{n\epsilon^2}{8\sigma^4} \right ), \textrm { and }
  P \left (  \hat\sigma_n <  \sigma - \epsilon   \right) \! \leq \!
 \exp \left (  - \frac{n\epsilon^2}{8\sigma^4}\right ), \textrm{ for any } \epsilon \ge 0.
\end{align*} 
\end{lemma}
\begin{proof}
Consider $Z=(Z_1,\ldots,Z_n)$ a vector of i.i.d. standard Gaussian r.v.s and a $L$-Lipschitz function  $f : \R^n \rightarrow \R$. Then, using Gaussian concentration for Lipschitz functions (cf. Section 2.3 of \cite{wainwright2019high}) 
\begin{align} 
\Prob{ f(Z) - \E f(Z) > \epsilon} \le \exp \left (  - \frac{n\epsilon^2}{2L^2} \right ).
\label{eq:gauss-conc-lipschitz}
\end{align}
For $i=1,\ldots,n,$ with i.i.d. Gaussian r.v.s $X_i \sim {\cal N}(0, \sigma)$, 
consider $Z_i=\frac{X_i}{\sigma}$ and $f(z^n) \triangleq \sqrt{\frac{1}{n} \sum \limits_{i=1}^{n} z_{i}^{2}}, \ z^n \in \mathbbm{R}^n$. Observing that $f$ is $1$-Lipschitz, changing the variable from $Z$ to $X=(X_1,\ldots,X_n)$ and using \eqref{eq:gauss-conc-lipschitz},  we obtain
\[P \left ( \hat\sigma_n >  \sigma + \epsilon   \right)  \leq 
 \exp \left (  - \frac{n\epsilon^2}{2\sigma^2} \right ).\]
 The other inequality bounding the left tail follows by an argument similar to above.
 
For the case of sub-Gaussian r.v.s, the main claim can be inferred from Theorem 3.1.1 in \cite{vershynin2016high} and we provide the proof details below for the sake of completeness. 
Observe that
\begin{align*}
\Prob{\sqrt{\frac{1}{n} \sum \limits_{i=1}^{n} Z_{i}^{2}} - 1 > \epsilon}
&\le \Prob{\frac{1}{n} \sum \limits_{i=1}^{n} Z_{i}^{2} - 1 > \max(\epsilon,\epsilon^2)}\le \exp \left (  - \frac{n \epsilon^2}{8}\right )
\end{align*}
The first inequality above holds because $x - 1 \ge \epsilon$ implies $x^2 - 1 \ge \max(\epsilon,\epsilon^2)$, for any $\epsilon\ge 0$, while the final inequality follows from Lemma \ref{lemma:subexp-conc} after observing that $Z_i^2$ is sub-exponential since $Z_i$ is sub-Gaussian.
The main claim follows by changing the variable to $X$ from $Z$. As before, the other inequality bounding the left tail follows by a completely parallel argument.
 \end{proof}

Next, we state and prove a result that establishes exponential concentration of the sample correlation coefficient.  
\begin{lemma}\textbf{\textit{(Concentration of sample correlation coefficient)}} 
\label{lemma:rho-conc}
For independent Gaussian rvs $X_{i}, \ i=1,\ldots,n$, with mean zero and covariance matrix $\Sigma$ as defined in \eqref{eq:gauss-covar} and with $\hat\sigma_i^2$, $\hat\rho_{ij}$ formed from $n$ samples using \eqref{eq:rhohat}, for any $i,j=1,\ldots,K$, and for any $\epsilon \in [0,\eta]$,  we have 
\begin{align*}
 & \P \left (  \left|\hat\rho_{ij} -  \rho_{ij} \right| > \epsilon   \right) 
\le 26\exp \left ( - \frac{n}{8} \frac{1}{36(1+\eta)} \min\left( \frac{l \epsilon}{3}, \left ( \frac{l \epsilon}{3} \right)^2\right) \right), 
\end{align*} 
where $l$ is a positive constant satisfying $l \le \sigma_i^2 \le 1$, $\forall i$.
\end{lemma}
\begin{proof}
We bound $\P \left (  \hat\rho_{ij} < \rho_{ij} - \epsilon   \right)$ for $i=1$ and $j=2$. The analysis below holds in general. 

Consider the following event: 
\[\B = \left \{ \sigma_1^2 - \epsilon \leq {\hat \sigma}_{1}^2 \leq \sigma_{1}^2 + \epsilon, \ \sigma_2^2 - \epsilon \leq {\hat \sigma}_{2}^2 \leq  \sigma_{2}^2 + \epsilon, \sigma_1 - \epsilon \leq {\hat \sigma}_{1} \leq \sigma_{1} + \epsilon, \ \sigma_2 - \epsilon \leq {\hat \sigma}_{2} \leq  \sigma_{2} + \epsilon \right  \}.\] Using Lemmas \ref{lemma:subexp-conc}--\ref{lemma:sample-stddev-conc-bd}, 
\begin{align}
 \P \left ( \B^c \right ) 
 & \leq \P  \left ( \hat \sigma_1^2 >  \sigma_1^2 + \epsilon \right ) + \P  \left ( \hat \sigma_1^2 < \sigma_1^2 - \epsilon \right ) 
      + \P  \left ( \hat \sigma_2^2 >  \sigma_2^2 + \epsilon \right ) + \P  \left ( \hat \sigma_2^2 <  \sigma_2^2 - \epsilon \right ) \nonumber\\
      &\qquad \P  \left ( \hat \sigma_1 >  \sigma_1 + \epsilon \right ) + \P  \left ( \hat \sigma_1 < \sigma_1 - \epsilon \right ) 
      + \P  \left ( \hat \sigma_2 >  \sigma_2 + \epsilon \right ) + \P  \left ( \hat \sigma_2 <  \sigma_2 - \epsilon \right ) \nonumber\\
 & \le 4 \exp \left ( - \frac{n}{8} \min\left(\epsilon,\epsilon^2\right) \right ) + 4 \exp \left ( - \frac{n\epsilon^2}{8} \right ), \label{eq:eventB-bd}
\end{align}
where the penultimate inequality  relies on the assumption that $\sigma_1^2, \ \sigma_2^2 \le 1$.

Let $Y_{12n} = \frac{1}{(1-\rho_{12})} \left( \frac{\overline X_1^2}{\hat\sigma_1^2} + \frac{\overline X_2^2}{\hat\sigma_2^2} - 2 \frac{\overline{X_1X_2}}{\hat\sigma_1\hat\sigma_2} \right)$ and 
$\bar Y_{12 n } = \frac{1}{(1-\rho_{12})} \left( \frac{\overline X_1^2}{\sigma_1^2} + \frac{\overline X_2^2}{\sigma_2^2} - 2 \frac{\overline{X_1X_2}}{\sigma_1\sigma_2} \right)$. Then, on the event $\B$, we have
\begin{align*}
& \P \left (  \hat\rho_{12} <  \rho_{12} - \epsilon ,\B  \right) = \P \left ( \left ( 1 - \hat\rho_{12}  \right ) - \left ( 1 - \rho_{12}  \right ) > \epsilon, \B  \right )\\
 &= \P \left ( \left ( 1 - \rho_{12}  \right ) \frac{Y_{12n}}{2 } - \left ( 1 - \rho_{12}  \right ) > \epsilon, \B  \right ) \\
&\le \P \left ( \left ( 1 - \rho_{12}  \right )\left( \frac{\bar Y_{12n}}{2 }  - 1   \right ) > \frac{\epsilon}{2}, \B  \right ) \\
& \ \ \  + \P \left ( \frac{\overline X_1^2}{2 }\left ( \frac{1}{\hat\sigma_1^2} - \frac{1}{\hat\sigma_1^2} \right ) + \frac{\overline X_2^2}{2 }\left ( \frac{1}{\hat\sigma_2^2} - \frac{1}{\hat\sigma_2^2} \right ) - 2 \frac{\overline{X_1X_2}}{2 }\left ( \frac{1}{\hat\sigma_1\hat\sigma_2} - \frac{1}{\sigma_1\sigma_2} \right )> \frac{\epsilon}{2}, \B  \right ) \\
&\le \P \left ( \frac{\bar Y_{12n}}{2}  - 1    \geq \frac{\epsilon}{2\left ( 1 - \rho_{12}  \right )}  \right ) + \P \left ( \frac{\overline X_1^2}{2 }\left ( \frac{1}{\hat\sigma_1^2} - \frac{1}{\sigma_1^2} \right ) > \frac{\epsilon}{6}, \B  \right ) \\
& \ \ \  + \P \left ( \frac{\overline X_2^2}{2 }\left ( \frac{1}{\hat\sigma_2^2} - \frac{1}{\sigma_2^2} \right ) > \frac{\epsilon}{6}, \B  \right )+ \P \left (  {\overline{X_1X_2}}\left ( - \frac{1}{\hat\sigma_1\hat\sigma_2} + \frac{1}{\sigma_1\sigma_2} \right ) > \frac{\epsilon}{6}, \B  \right )  \stepcounter{equation}\tag{\theequation}\label{eq:t122}\\
\end{align*}
We now bound each term on the RHS above. 
The first term in \eqref{eq:t122} is bounded by an application of Lemma \ref{lemma:subexp-conc} as follows:
\begin{align*}
 \P \left ( \frac{\bar Y_{12n}}{2}  - 1    \geq \frac{\epsilon}{2\left ( 1 - \rho_{12}  \right )}  \right )  &
 \le \P \left ( \frac{\bar Y_{12n}}{2}  - 1    \geq \frac{\epsilon}{4}  \right ) \quad \textrm{(Since } |\rho_{ij}| \le 1) \\
& \le 
 \concsubexpTwoterms{n}{\frac{\epsilon}{2}}
\end{align*}
The second term in \eqref{eq:t122} is bounded as follows:
\begin{align*}
&\P \left ( \frac{\overline X_1^2}{2}\left ( \frac{1}{\hat\sigma_1^2} - \frac{1}{\sigma_1^2} \right ) > \frac{\epsilon}{6}, \B  \right ) 
 \le \P \left (  \frac{\sigma_1^2  - \hat\sigma_1^2}{ \sigma_1^2 }  \ge \frac{\epsilon}{3 }  \right ) \le \P \left (  \sigma_1^2  - \hat\sigma_1^2  > \frac{l\epsilon}{3 }  \right ) \\
 \leq & \concsubexpTwoterms{n}{\frac{l\epsilon}{3}}
\end{align*}
where the penultimate inequality follows from an application of Lemma \ref{lemma:subexp-conc} and the last inequality uses the fact that $0< l <1$. 
The third term in \eqref{eq:t122} can be bounded in a similar fashion. 
The last term in \eqref{eq:t122} is bounded as follows:
\begin{align*}
&\P \left (   {\overline{X_1X_2}}\left ( -\frac{1}{\hat\sigma_1\hat\sigma_2} + \frac{1}{\sigma_1\sigma_2} \right ) > \frac{\epsilon}{3}, \B  \right )\\
& \le \P \left (  \hat\sigma_1( \hat\sigma_2  - \sigma_2) + \sigma_2( \hat\sigma_1  -  \sigma_1)  > \frac{\epsilon\hat\sigma_1\sigma_1\hat\sigma_2\sigma_2}{3  \overline{X_1X_2}}, \B  \right ) \\
 & \le  \P \left (  \hat\sigma_1( \hat\sigma_2  - \sigma_2) > \frac{\epsilon\hat\sigma_1\sigma_1\hat\sigma_2\sigma_2}{6  \overline{X_1X_2}}, \B \right ) +  \P \left ( \sigma_2( \hat\sigma_1  -  \sigma_1)  > \frac{\epsilon\hat\sigma_1\sigma_1\hat\sigma_2\sigma_2}{6  \overline{X_1X_2}}, \B  \right )  \stepcounter{equation}\tag{\theequation}\label{eq:t124}\\
 & \le  \P \left (  \hat\sigma_1( \hat\sigma_2  - \sigma_2) > \frac{\epsilon l}{6 }, \B \right ) +  \P \left (  \hat\sigma_1  -  \sigma_1  > \frac{\epsilon \sqrt{l}}{6 }, \B  \right ) \quad  \textrm{(since } {\overline{X_1X_2}} < \hat\sigma_1\hat\sigma_2, \textrm { and } \sigma_i^2 \ge l)   \\ 
 & \le  \P \left (   \hat\sigma_2  - \sigma_2 > \frac{\epsilon l}{6\sqrt{u+\eta} }, \B   \right )  
 + \P \left (  \hat\sigma_1  - \sigma_1  > \frac{\epsilon\sqrt{l}}{6} , \B   \right ) 
  ,  \textrm{(since } \hat\sigma^2_i\le \sigma^2_1 + \epsilon \le 1 + \eta,  \textrm{ on } \B \textrm{ and for } \epsilon < \eta)  \\
 &\le \exp \left ( - \frac{n\epsilon^2 l^2}{8\times 36(1+\eta)}  \right) 
 + \exp \left ( - \frac{n\epsilon^2 l}{8\times 36}   \right)
 \quad \textrm{ (Lemma \ref{lemma:sample-stddev-conc-bd})}\\
 & \le 2\exp \left ( - \frac{n\epsilon^2 l^2}{8\times 36(1+\eta)}  \right).
\end{align*}
Using the bounds obtained above for each of the terms on the RHS of \eqref{eq:t122} in conjunction with the bound on $\Prob{\B^c}$ in \eqref{eq:eventB-bd}, we obtain
\begin{align*}
 &\P \left (  \hat\rho_{12} <  \rho_{12} - \epsilon   \right)  \le  \P \left (  \hat\rho_{12} <  \rho_{12} - \epsilon , \B  \right) 
+ \Prob{\B^c} \\
 &\qquad\le \concsubexpTwoterms{n}{\frac{\epsilon}{2}} + 2\concsubexpTwoterms{n}{\frac{l\epsilon}{3}} + 2\exp \left ( - \frac{n\epsilon^2 l^2}{8\times 36(1+\eta)}  \right)\\
 &\qquad\quad +  4 \exp \left ( - \frac{n}{8} \min\left(\epsilon,\epsilon^2\right) \right ) + 4 \exp \left ( - \frac{n\epsilon^2}{8} \right ).
 \end{align*}
 
 Now, it is easy to see that
 \begin{align*}
   &  \exp \left ( - \frac{n}{8} \min\left(\epsilon,\epsilon^2\right) \right ) \leq \concsubexpTwoterms{n}{\frac{\epsilon}{2}}, \\
  & \concsubexpTwoterms{n}{\frac{\epsilon}{2}} \leq \concsubexpTwoterms{n}{\frac{l\epsilon}{3}}, \quad \text{ since $l \leq 1.$ } 
 \end{align*}
Hence, we obtain 
 \begin{align*}
     \P \left (  \hat\rho_{12} <  \rho_{12} - \epsilon   \right)  \le 7 \concsubexpTwoterms{n}{\frac{l\epsilon}{3}} + 4 \exp \left ( - \frac{n\epsilon^2}{8} \right ) + 2\exp \left ( - \frac{n\epsilon^2 l^2}{8\times 36(1+\eta)}  \right).
 \end{align*}
 Noting that
 \begin{align*}
  &   -\epsilon^2 \leq 
  - \min \left( \frac{l \epsilon}{3}, \left (\frac{l \epsilon}{3} \right)^2 \right) 
  \leq - \frac{1}{36(1+\eta)} \min  \left( \frac{l \epsilon}{3}, \left (\frac{l \epsilon}{3} \right)^2 \right)
  \text{since $l \leq 1,$ }
 \end{align*}
 we have 
 \begin{align*}
\P \left (  \hat\rho_{12} <  \rho_{12} - \epsilon   \right)  \ \le 13\exp \left ( - \frac{n}{8} \frac{1}{36(1+\eta)} \min\left( \frac{l \epsilon}{3}, \left ( \frac{l \epsilon}{3} \right)^2\right) \right).
\end{align*}

Thus, we have proved the claim concerning the (left tail) deviations of the sample correlation coefficient. The second claim concerning the right tail can be proved in a similar fashion.
\end{proof}

As mentioned before, the MSE estimate in \eqref{eq:mse-est-i} is comprised of sample variances and sample correlation coefficients. To prove that the MSE estimate concentrates, we shall use Lemma \ref{lemma:subexp-conc} for terms involving sample variances, and the lemma below for terms involving sample correlation coefficients.
\begin{proof}\textbf{(Proposition \ref{prop:mse-conc})}\ \\
We prove the proposition for $i=1$, but the analysis below holds in general. 
Consider the events \\$\C = \left \{ \sigma_i^2 - \epsilon \leq {\hat \sigma}_{i}^2 \leq  \sigma_{i}^2 + \epsilon, i=1,\ldots,K    \textrm{ and } \rho_{1j} - \epsilon \le \hat\rho_{1j} \le \rho_{1j}+ \epsilon,  \textrm{ for } j=2,\ldots,K \right\}$.
Then, from Lemmas \ref{lemma:subexp-conc}--\ref{lemma:rho-conc}, we have
\begin{align*}
 \Prob{ \C^c} &\le 2K \exp \left ( - \frac{n}{8} \min\left(\epsilon,\epsilon^2\right)   \right) + 26K\exp \left ( - \frac{n}{8} \frac{1}{36(1+\eta)} \min\left( \frac{l \epsilon}{3}, \left ( \frac{l \epsilon}{3} \right)^2\right) \right)\\ 
 & \le 28K\exp \left ( - \frac{n}{8} \frac{1}{36(1+\eta)} \min\left( \frac{l \epsilon}{3}, \left ( \frac{l \epsilon}{3} \right)^2\right) \right).
\end{align*}
We shall bound the tail probability $\Prob{ \hat\cE_{1} - \cE_{1} > \epsilon}$ on the event $\C$ and use the bounds on the probability of $\C^c$ to arrive at an unconditional bound on the aforementioned tail probability.
Using an union bound, we have
\begin{align*}
 &  \P \left (  \hat\cE_1  - \cE_1 \geq  \epsilon \! \right ) \le \P \left ( \hat\cE_1  - \cE_1 \geq \epsilon, \C \! \right ) + 
 \P \left ( \C^c \! \right ) \\
 & \leq \P \left ( {\hat \sigma}_{2}^2 \left ( 1 - \hat\rho_{12}^2  \right ) - { \sigma}_{2}^2 \left ( 1 - \rho_{12}^2  \right )  \geq \frac{\epsilon}{(K-1)}, \C  \right ) 
    + \sum \limits_{p=3}^{K}  \P \left (  {\hat \sigma}_{p}^2 \left ( 1 - \hat\rho_{1p}^2  \right ) - { \sigma}_{p}^2 \left ( 1 - \rho_{1p}^2  \right )  \geq \frac{\epsilon}{(K-1)}, \C \right )\\
    &\qquad + \P \left( \C^c \right).  
\stepcounter{equation}\tag{\theequation}\label{eq:allinone}
\end{align*}
The first term on the RHS  above can be bounded as follows:
\begin{align*}
&\P \left ( {\hat \sigma}_{2}^2 \left ( 1 - \hat\rho_{12}^2  \right ) - { \sigma}_{2}^2 \left ( 1 - \rho_{12}^2  \right ) \geq \frac{\epsilon}{(K-1)}, \C  \right )\\
&\le \P \left ( \hat\sigma_{2}^2 \left(\left ( 1 - \hat\rho_{12}^2  \right ) - \left ( 1 - \rho_{12}^2  \right )\right) \geq \frac{\epsilon}{2(K-1)}, \C  \right ) 
 + \P \left ( \left ( 1 - \rho_{12}^2  \right )\left ( \hat\sigma_{2}^2 - \sigma_{2}^2  \right ) \geq \frac{\epsilon}{2(K-1)}, \C  \right ) \\
&\le  \P \left ( \rho_{12}^2  - \hat\rho_{12}^2   \geq \frac{\epsilon}{2(K-1)(1+\eta)}, \C  \right ) 
 + \P \left (  \hat\sigma_{2}^2 - \sigma_{2}^2   \geq \frac{\epsilon}{2(K-1)}  \right ) \textrm{(since } \hat\sigma^2_2\le \sigma^2_2 + \epsilon \le 1 + \eta,  \textrm{ on } \C)
\\
 &\le \P \left (  \rho_{12}- \hat\rho_{12} \geq \frac{\epsilon}{4(K-1)(1+\eta)}, \C  \right ) 
 + \P \left (  \hat\sigma_{2}^2 - \sigma_{2}^2   \geq \frac{\epsilon}{2(K-1)}  \right ) 
 \textrm{ (since } \hat\rho_{12}, \rho_{12} \le 1)\\
 & \le 13\exp \left ( - \frac{n}{8} \frac{1}{36(1+\eta)}   \min\left( \frac{l\epsilon}{3*4*(K-1)*(1+\eta)}, \left ( \frac{l\epsilon}{3*4*(K-1)*(1+\eta)} \right )^2 \right) \right) \\
 & \qquad +  \exp \left ( - \frac{n}{8}  \min\left( \frac{\epsilon}{2(K-1)}, \frac{\epsilon^2}{4(K-1)^2}\right) \right) \textrm{(From Lemmas  \ref{lemma:subexp-conc}--\ref{lemma:rho-conc})}\\
 & \le 14\exp \left ( - \frac{n}{8} \frac{1}{36(1+\eta)}   \min\left( \frac{l\epsilon}{3*4*(K-1)*(1+\eta)}, \left ( \frac{l\epsilon}{3*4*(K-1)*(1+\eta)} \right )^2 \right) \right) \\
 & \le 14\exp \left ( - \frac{n}{8} \frac{1}{36(1+\eta)}   \min\left( \frac{l\epsilon}{12*K*(1+\eta)}, \left ( \frac{l\epsilon}{12*K*(1+\eta)} \right )^2 \right) \right) \\
 & \le 14\exp \left ( - \frac{n}{8} \frac{1}{36(1+\eta)}   \left ( \frac{l\epsilon}{12*K*(1+\eta)} \right )^2  \right) \qquad \text{ for $\epsilon \leq 2K$} \\
 & \le 14 \exp \left ( - \frac{n}{8} \frac{1}{36(1+2K)}   \left ( \frac{l\epsilon}{12*K*(1+2K)} \right )^2  \right), \qquad \text{ for $\eta \leq 2K.$ }
   \stepcounter{equation}\tag{\theequation}\label{eq:term1-1}
\end{align*}
In the above, it suffices to look at $\epsilon \leq 2(K-1)$, since each $\cE_i$ itself is less than $K-1$, as $\sigma_i^2 \leq 1$. Consequently, it is sufficient to consider $\eta = 2K$.

A bound similar to that in \eqref{eq:term1-1} can be  obtained for the other terms inside the summation in \eqref{eq:allinone}, leading to
\begin{align*}
   \P \left (  \hat\cE_1  - \cE_1 \geq  \epsilon \! \right ) &\le 14 (K-1) \exp \left ( - \frac{n}{8} \frac{1}{36(1+2K)}   \left ( \frac{l\epsilon}{12*K*(1+2K)} \right )^2  \right) \\
   & +  28K\exp \left ( - \frac{n}{8} \frac{1}{36(1+\eta)} \min\left( \frac{l \epsilon}{3}, \left ( \frac{l \epsilon}{3} \right)^2\right) \right)\\
 & \le 42 K  \exp \left ( - \frac{n}{8} \frac{1}{36(1+2K)}   \left ( \frac{l\epsilon}{12*K*(1+2K)} \right )^2  \right), \qquad \text{for $\epsilon \leq 2K$ } \\
 & \le 42 K  \exp \left ( - \frac{n}{8} \frac{1}{108 K}   \left ( \frac{l\epsilon}{36K^2} \right )^2  \right)   .
\end{align*}

A concentration inequality to the bound the  lower semi-deviations can be derived  in a similar fashion. Hence proved.
\end{proof}